\let\latexarabic\arabic
\let\latexdocument\document
\let\latexenddocument\enddocument

\RequirePackage[thmmarks]{ntheorem}
\makeatletter
\renewtheoremstyle{plain}
  {\item[\hskip\labelsep \theorem@headerfont ##1\ \textup{##2}\theorem@separator]}
  {\item[\hskip\labelsep \theorem@headerfont ##1\ \textup{##2}\ (##3)\theorem@separator]}
\makeatother

\documentclass[lineno, letter]{biometrika}
 
\usepackage[OT1]{fontenc}

\usepackage[normalem]{ulem} 
\usepackage{todonotes}

\let\document\latexdocument
\let\enddocument\latexenddocument
\AtEndDocument{\printhistory}
\let\arabic\latexarabic
\def\rm{}

\usepackage{subcaption}
\usepackage{url}
\usepackage{amsmath}
\usepackage{todonotes}
\usepackage [english]{babel}
\usepackage [autostyle, english = american]{csquotes}
\MakeOuterQuote{"}
 
\usepackage{times}
\usepackage{bm}
\usepackage{natbib}
\usepackage{amsfonts}

\usepackage[plain,noend]{algorithm2e}

\makeatletter
\renewcommand{\algocf@captiontext}[2]{#1\algocf@typo. \AlCapFnt{}#2} 
\def\@algocf@capt@plain{top}
\renewcommand{\algocf@makecaption}[2]{%
  \addtolength{\hsize}{\algomargin}%
  \sbox\@tempboxa{\algocf@captiontext{#1}{#2}}%
  \ifdim\wd\@tempboxa >\hsize
    \hskip .5\algomargin%
    \parbox[t]{\hsize}{\algocf@captiontext{#1}{#2}}
  \else%
    \global\@minipagefalse%
    \hbox to\hsize{\box\@tempboxa}
  \fi%
  \addtolength{\hsize}{-\algomargin}%
}
\makeatother

\newcommand{\R}{\mathbb{R}}

\begin{document}

\jname{Manuscript}





\markboth{E. Tam and D.B. Dunson}{On the Statistical Capacity of Deep Generative Models}

\title{On the Statistical Capacity of Deep Generative Models}

\author{Edric Tam}
\affil{Department of Biomedical Data Science, Stanford University,\\
300 Pasteur Drive, 
Stanford, California 94305, U.S.A. \email{edrictam@stanford.edu}}

\author{\and David B. Dunson}
\affil{Department of Statistical Science and Department of Mathematics, Duke University, \\Box 90251
Durham, North Carolina 27708, U.S.A. \email{dunson@duke.edu}}

\maketitle

\begin{abstract}
Deep generative models are routinely used in generating samples from complex, high-dimensional distributions. 
Despite their apparent successes, their statistical properties are not well understood. A common assumption is that with enough training data and sufficiently large neural networks,
deep generative model samples 
will have arbitrarily small errors in sampling from any continuous target distribution. 
We set up a unifying framework that debunks this belief. We demonstrate that broad classes of deep generative models, including variational autoencoders and generative adversarial networks, are not universal generators. Under the predominant case of Gaussian latent variables, these models can only generate concentrated samples that exhibit light tails. Using tools from concentration of measure and convex geometry, we give analogous results for more general log-concave and strongly log-concave latent variable distributions. We extend our results to diffusion models via a reduction argument. We use the Gromov--Levy inequality to give similar guarantees when the latent variables lie on manifolds with positive Ricci curvature. These results shed light on the limited capacity of common deep generative models to handle heavy tails. We illustrate the empirical relevance of our work with simulations and financial data. 
\end{abstract}

\begin{keywords}
generative adversarial networks, variational autoencoders, diffusion models, manifold hypothesis, concentration of measure, isoperimetric inequalities
\end{keywords}

\section{Introduction}

A fundamental task in statistics is to generate samples $x$ from a target probability distribution $\pi$. When $\pi$ has an explicitly specified density up to normalization, often the case in Bayesian modeling, Markov chain Monte Carlo samplers are the gold standard. However, in modern applications involving complex data such as images and natural language, $\pi$ is often too complicated and high-dimensional to be explicitly stated. Instead, the target distribution $\pi$ is implicitly specified via a collection of independent training samples $\tilde{x}$. Learning to sample from these implicit targets is known as "generative modeling" in the machine learning literature. 

Deep generative models are  related to latent variable models in the probabilistic and Bayesian modeling literature, with deep neural networks used in defining mappings from latent variables to observed data. The core idea is to transform latent variables $z$ with a function $f$ so that the law of $f(z)$ approximates the target $\pi$. Deep neural networks $\hat{f}$, given their immense flexibility, are natural candidates for modeling $f$. 
A variety of loss functions have been proposed for fitting $\hat{f}$, with motivations ranging from adversarial considerations \citep{goodfellow2020generative} to variational inference \citep{kingma2014auto}. To generate approximate samples from 
$\pi$, one simply applies the fitted $\hat{f}$ to realizations of $z$. One can further consider sequentially transforming $z$ using multiple neural networks, as in diffusion models \citep{ho2020denoising}. 

The vast majority of existing work in the deep generative modeling literature impose Gaussian distributions on the latent variables $z$ \citep{rezende2014stochastic, kingma2014auto}. Owing to the status of neural networks as universal function approximators \citep{cybenko1989approximation, barron1993universal, hornik1991approximation}, there is a folklore that deep generative models enjoy similarly rich expressivity \citep{doersch2016tutorial, kingma2019introduction}. It is widely assumed that, given enough training data and sufficiently large neural networks, such transformation-based deep generative models will have arbitrarily small approximation error for any continuous target distribution, even when the latent variable distributions are chosen to be simple \citep{hu2018stein}. 

Our work here debunks this belief. We start by showing that for Gaussian latent variables $z$, the law of $\hat{f}(z) - E\{\hat{f}(z)\}$ is light-tailed. This demonstrates that deep generative models such as generative adversarial networks and variational autoencoders are not universal generators in practice. This also shows that the common practice of defaulting to Gaussian latent variables is not always appropriate. We generalize in several directions. First, we show analogous results for log-concave and strongly log-concave latent variables $z$. Second, we give similar guarantees when the latent variables $z$ lie on a manifold with positive Ricci curvature. Third, we extend our results to denoising diffusion models by using a reduction argument. Many of our results are dimension-free, in the sense that the bounds obtained do not explicitly depend on the dimension of the latent variables. None of our results resort to asymptotic approximations. 

Our work shows that a broad class of common deep generative models are not universal generators. Since the center of the learned distribution of $\hat{f}(z)$ remains completely flexible, it is unsurprising that a typical sample from such deep generative models empirically resembles typical samples from the target distribution. However, due to the light-tailedness of the law of $\hat{f}(z)- E\{\hat{f}(z)\}$, when the target distribution is heavy-tailed, samples from such deep generative models will tend to underestimate the uncertainty and diversity of the true distribution. This has substantial implications for practitioners. For one, deep generative models are commonly adopted in anomaly detection \citep{schlegl2017unsupervised} and finance \citep{eckerli2021generative}, applications where tails play a crucial role. For another, there is an emerging interest in the Bayesian literature in leveraging various generative models for posterior sampling \citep{polson2023generative, winter2024emerging}, a setting in which underestimating uncertainty can lead to incorrect downstream inference. 

\subsection{Related work}
There is a broad literature on deep generative models. See \cite{bond2021deep} for a review. 
There is a common impression that such models are extremely expressive \citep{kingma2019introduction, doersch2016tutorial, hu2018stein}. There is a literature \citep{lu2020universal, yang2022capacity} that offers universal approximation theorems for deep generative models under moment conditions using metrics such as the Wasserstein distance. Research on the theoretical limitations of deep generative models is relatively scarce. It has been observed that variational autoencoders and generative adversarial networks have difficulty modeling multi-modal distributions \citep{salmona2022can}. \cite{wiese2019copula} studies the limitations of certain deep generative models from a tail asymptotics perspective. \cite{oriol2021some} gives limitations of Gaussian generative adversarial networks when the output is one-dimensional. 

\section{Preliminaries}

\subsection{Deep neural networks}

We consider feed-forward neural networks of depth $L$. Given input $z \in \R^d$, define the network via the composition $\hat{f}(z) = h_L[h_{L-1}\{\ldots h_1(z) \ldots\}]$, where $h_l(z) = \sigma_l(W_l z + b_l)$, $\sigma_l$ is a non-linear activation function operating elementwise on the $l$th layer, and $W_l$ and $b_l$ are respectively the weight matrix and bias vector corresponding to the $l$th layer. This setup allows the dimensions of $W_l$, as well as the choice of activation functions, to vary between layers. Let $\text{width}(W_l)$ denote the maximum of the number of rows and columns of 
$W_l$, and $\max_{l = 1}^L \text{width}(W_l)$ denote the width of the neural network. For additional information, see the excellent review by \cite{fan2021selective}.

We use $d$ to denote latent variable dimension and $p$ to denote output dimension. Let $\hat{f}:\R^d \to \R^p$ denote the trained neural network function used for sample generation. The function $\hat{f}$ is Lipschitz if $\sup_{x, y \in \R^d} ||\hat{f}(x) - \hat{f}(y)||_2/||x - y||_2 \leq \mathcal{L}$ for some $\mathcal{L} > 0$, where $||\cdot||_2$ denotes the Euclidean norm.
Letting $\mathcal{S}$ denote the set of all Lipschitz activation functions, $S$ includes common choices in practice \citep{virmaux2018lipschitz}, including the rectified linear unit function $\sigma_{ReLU}(x) = \max(0, x)$, the logistic function $\sigma_{logistic}(x) = \{1+ \exp(-x)\}^{-1}$, the hyperbolic tangent function $\tanh(x)$, and beyond. 
We define finite feed-forward neural networks below.

\begin{definition}[Finite feed-forward neural networks]
    A feed-forward neural network is finite if 
        (1) the depth $L$ is finite, 
        (2) the width $\max_{l = 1}^L \text{width}(W_l)$ is finite, 
        (3) all entries in the matrices $W_{l = 1}^L$ and vectors $b_{l = 1}^L$ are finite, and
        (4) all activation functions $\sigma_{l = 1}^L$ are members of $\mathcal{S}$.
    We denote the set of all finite feed-forward neural networks as $\mathcal{F}$. 
\end{definition}


This notion of finity encompasses most feed-forward neural networks used in practice.

\begin{proposition} \label{finite_lipschitz}
Finite feed-forward neural networks are Lipschitz with respect to the Euclidean norm. 
\end{proposition} 

\begin{remark}
Many popular neural network operations, such as dropout, pooling and batch normalization, have finite Lipschitz constants. Our results can be extended to a generalized function class that incorporates a finite number of these Lipschitz operations. 
\end{remark}

\subsection{Deep generative modeling}

Consider the following latent variable model. 
\begin{align}
    x_i &= f(z_i) + \epsilon_i,\quad 
    z_i \sim P,\quad 
    \epsilon_i \sim Q, \nonumber 
\end{align}
where $x_i$ is the observed data for sample $i$, which is equal to a function $f$ of a latent variable $z_i$ plus an additive noise $\epsilon_i$. The latent variable distribution $P$ and noise distribution $Q$ are often chosen to be multivariate Gaussian with diagonal covariance.
Linear $f$ leads to classical Gaussian factor models, while using a deep neural net for $f$ provides the foundation of broad classes of deep generative models. In the next section, we give general theoretical results on the law of $\hat{f}(z)$ that hold for any $\hat{f} \in \mathcal{F}$. 

\section{Isoperimetry and Concentration of Deep Generative Models}

The notion of concentration of measure is central to the development below. Related definitions of sub-Gaussian and sub-exponential random vectors are reviewed in the Supplementary Materials. To ease notation, throughout the paper we follow the convention where we use $C, c > 0$ to denote absolute constants whose values are unspecified. We use subscripts like $C_p$ to highlight any dependencies. After training, the fitted neural network $\hat{f} \in \mathcal{F}$ at the generation phase is a fixed function with a finite Lipschitz constant. The output dimension $p$ and latent dimension $d$ are fixed constants here. We do not make any attempts to optimize constant factors in any inequalities below. We use the notation $S^{p - 1}$ to denote the unit $(p-1)$-sphere in $\R^p$. 

We start with a result on deep generative models with Gaussian latent variables, the predominant case in the literature. 

\begin{theorem}[Deep Generative Models with Gaussian Latent Variables] \label{gaussian_iso}
Let $z$ be a Gaussian random vector with mean $\mu$ and covariance $\Sigma$. Let $\hat{f}: \R^d \to \R^p$ be any finite neural network function with Lipschitz constant $\mathcal{L}$. Then for any unit vector $u \in S^{p - 1}$,  
$Pr(|u^T [\hat{f}(z) - E\{\hat{f}(z)\}]| \geq t) \leq 2\exp(-t^2/C_p^2)$
where $C_p^2 = 
 C^2p\mathcal{L}^2||\Sigma||$ and $C >0$.
\end{theorem}

The above theorem, which relies on the well-known Gaussian isoperimetric inequality, implies that $\hat{f}(z) - E\{\hat{f}(z)\}$ is sub-Gaussian. Since sub-Gaussian distributions have light tails, they are inappropriate for modeling heavy-tailed distributions. Since this result can be applied to any member of $\mathcal{F}$, this limitation cannot be overcome by increasing training data or enlarging the neural network. Since we are chiefly interested in the tail behavior of the generated samples, rather than the location of the mean, this centred quantity is appropriate for our context. The mean $E\{\hat{f}(z)\}$ in the above result can be
replaced by the median with only changes to universal constants \citep{wainwright2019high}.  

While the Gaussian latent variables case is the most prevalent, a variety of alternative easy-to-sample  latent variable distributions have been considered. We give analogous theoretical results on log-concave latent variables. Log-concave distributions are a broad family that include the important case of uniform distributions on any convex body, such as the hypercube and hyperball. 

\begin{theorem}[Deep Generative Models with  Log-concave latent variables]\label{logconcave_iso}
Let $z \in \R^d$ be a log-concave random vector with covariance $\Sigma$. Let $\hat{f}: \R^d \to \R^p$ be any finite neural network with Lipschitz constant $\mathcal{L}$. Then for any $u \in S^{p-1}$ and $t \geq 0$, we have
$$\Pr(|u^T[\hat{f}(z) - E\{\hat{f}(z)\}]| \geq t) \leq 2\exp(-t/C_p)$$
for $C_p = C\sqrt{p}\mathcal{L}||\Sigma^{1/2}||/\Psi_z$, where $\Psi_z$ is the Cheeger's constant of the density of $z$ and $C > 0$. 
\end{theorem}

The above theorem implies that $|\hat{f}(z) - E\{\hat{f}(z)\}|$ is a sub-exponential random vector, which means it is also light-tailed, albeit less so than a sub-Gaussian. Theorem \ref{logconcave_iso} leverages tools from high-dimensional geometry \citep{lee2018kannan, gromov1983topological}. Notably, recent progress in the area \citep{chen2021almost, jambulapati2022slightly, klartag2022bourgain} demonstrates that the Cheeger's constant involved in the above upper bound can be replaced by a poly-logarithmic factor of the input dimension. 

Further variations, such as exponential-tilted Gaussian latent variables, have been proposed in the literature for applications such as out-of-distribution detection \citep{floto2023tilted}. These kinds of latent variables are strongly log-concave, for which sub-Gaussian bounds are available. 

\begin{theorem}
[Strongly Log-concave Lipschitz concentration]\label{strong_logconcave_iso}
Let $z$ be a $\gamma$-strongly log-concave random vector with covariance $\Sigma$. Let $\hat{f}: \R^d \to \R^p$ be any finite neural network with Lipschitz constant $\mathcal{L}$. Then for any unit vector $u \in S^{p - 1}$ we have 
$Pr(|u^T [\hat{f}(z) - E\{\hat{f}(z)\}]| \geq t] \leq 2\exp(-t^2/C_{p, \gamma}^2)$ where 
$C_{p, \gamma}^2 = C^2p\mathcal{L}^2||\Sigma||/\gamma$ and $C >0$.
\end{theorem}

This result again shows that $\hat{f}(z) - E\{\hat{f}(z)\}$ is a sub-Gaussian random vector. The above bounds for Gaussian and strongly log-concave latent variables do not explicitly depend on latent variable dimension $d$. This phenomenon is known as dimension-free concentration in the probability literature. In the case of log-concave latent variables, the bound's dependence on $d$ is poly-logarithmic. If a mathematical conjecture known as the Kannan--Lov\'asz--Simonovits conjecture is true
\citep{lee2018kannan}, even this small poly-logarithmic dependence on $d$ can be removed. 

\subsection{Manifold Setting}
Thus far, we have considered latent random variables that lie in Euclidean space. There are multiple other approaches that place latent variables on non-Euclidean manifolds, such as hyper-spheres
\citep{davidson2018hyperspherical}. Hence, we consider related results for deep generative models under the manifold setting. A particular property on manifolds that yields strong concentration behavior is positive Ricci curvature, with the canonical example being the hypersphere. We use the Gromov--Levy inequality from geometry to study the behavior of deep generative models when the latent variables come from such manifolds.  

We detail the main setting here. Let $(M, g)$ be a compact, connected $d_{int}$-dimensional Riemannian manifold with $d_{int} \geq 2$. Let $\lambda$ denote the infimum of the Ricci curvature tensor evaluated over any pair of unit tangent vectors associated with any point on the manifold and assume $\lambda > 0$. Letting $\nu$ be the corresponding normalized volume element, assume $z \sim \nu$. We consider the setting where $M$ is embedded in an ambient Euclidean space $\R^{d_{ext}}$. 
We assume that the embedding map $\phi:M\to \R^{d_{ext}}$ is Lipschitz with respect to the geodesic distance $D_{geo}$, so that $\sup_{a, b \in M} ||\phi(a) - \phi(b)||_2 \leq \mathcal{L} D_{geo}(a, b)$. 
This can be interpreted as a condition that controls the distortion of the geodesic distance structure when performing the embedding. 

To concretely illustrate the above setting, consider the $(d-1)$-hypersphere $rS^{d - 1}$ of radius $r$ naturally embedded in $\R^d$. It is a compact and connected manifold with $d_{int} = d - 1$, $d_{ext} = d$ and constant positive Ricci scalar curvature. $z \sim \nu$ implies $z$ is uniformly distributed on the hypersphere. The Lipschitz property is verified by observing that for any $x, y \in rS^{d - 1} \subset \R^d$, the geodesic distance $D_{geo}(x, y) = r\,  \textnormal{arccos}(x^Ty/r^2)$ upper bounds the Euclidean distance $||x- y||_2$. We now state our result.
 
\begin{theorem}[Concentration of latent variables on manifold] \label{manifold_iso}
Let $(M, g)$, $\nu$, $z$, $\lambda$ be defined as above. Let the embedding $\phi: M \to \R^{d_{ext}}$ be a $\mathcal{L}_\phi$-Lipschitz function with respect to the geodesic distance, and let $\hat{f}: \R^{d_{ext}} \to \R^p$ be any finite neural network function with Lipschitz constant $\mathcal{L}$. Then for any $u \in S^{p - 1}$, 
$$Pr(|u^T[\hat{f}\circ\phi (z) - E\{\hat{f}\circ\phi(z)\}]| \geq t) \leq 2\exp(-t^2/C_{\lambda}^2)$$
where $C_{\lambda}^2 = C^2p\mathcal{L}^2\mathcal{L}_\phi^2/\lambda$ and $C > 0$ is an absolute constant. 
\end{theorem}

The above result shows that the random vector $\hat{f}\circ\phi (z) - E\{\hat{f}\circ\phi(z)\}$ is sub-Gaussian. 

\section{Diffusion Models}

Diffusion models are important classes of deep generative models, with the denoising diffusion probabilistic model \citep{ho2020denoising} being one prominent example. Such models operate by modeling data generation via a diffusion process $(X_\tau)_{\tau = 0}^T$, with $X_\tau$ $p$-dimensional. One infers a 
reverse sampling process $X_T, X_{T-1},\ldots,X_0$ starting with a Gaussian latent variable $z = X_T \sim N_p(0, I)$ and performing a sequence of neural network transformations to generate the sample $X_0$ by iterating the following update step \citep{ho2020denoising} from $\tau = T,\ldots,1$: 
\begin{align} \label{iterative}
    X_{\tau - 1} &= (1/\sqrt{\alpha_\tau}) \{X_{\tau} - (1-\alpha_\tau)/\sqrt{1-\bar{\alpha}_\tau}\hat{f}(X_{\tau}, \tau)\} + \epsilon_\tau \sigma_{\tau} \text{1}_{\tau > 1},
\end{align}
where $(\alpha_\tau)$, $(\bar{\alpha}_\tau)$ and $(\sigma_\tau)$ are fixed sequences, $\epsilon_1, \ldots, \epsilon_T$ are independent $N_p(0, I)$ random vectors, $\text{1}_{\tau > 1}$ is an indicator function that prevents the sampler from adding noise on the last step, and $\hat{f}$ is a finite neural network that takes in $X_\tau$ and the time step $\tau$ as input.  

We develop a reduction argument, detailed in section \ref{reduction} in the supplementary materials, that allows us to treat the iterative transformations performed above equivalently as a single Lipschitz transformation on an augmented Gaussian random vector. This yields the following result for diffusion models with Gaussian latent variables. 

\begin{theorem}[Diffusion Models with Gaussian Latent Variables] \label{diffusion}
Let $X_0 \in \R^p$ be a sample generated from a denoising diffusion probabilistic model using the iterative procedure in \eqref{iterative}. Then for any unit vector $u \in S^{p - 1}$, there exists $\mathcal{L}_1, \ldots, \mathcal{L}_T > 0$ such that
$Pr[|u^T\{X_0  - E(X_0)\}| > t] \leq 2\exp(-t^2/C_p^2)$ where $C_p^2 = C^2p(\prod_{\tau = 1}^T\mathcal{L}_\tau)^2$ and $C > 0$. 
\end{theorem}

Here, $X_0  - E(X_0)$ is a sub-Gaussian random vector. We thus demonstrate that qualitatively, Gaussian diffusion models also suffer from light-tails, despite utilizing multiple transformations. The quantities $\mathcal{L}_1, \ldots, \mathcal{L}_T$ can intuitively be thought of as Lipschitz constants characterizing each iterative step of the sampling process. 

\section{Simulations and Data Illustration}

We assess the practical relevance of our theoretical results through simulations and data illustrations. We sampled $10000$ values from a bivariate Cauchy distribution with mode $0$ and scale matrix $I$. We then trained multiple generative adversarial networks with different depths and latent variable dimensions, as well as a denoising diffusion model on these data. We show the Cauchy training data and $10000$ samples from a four-layer generative adversarial network fitted with $64$ standard Gaussian latent variables in Figure \ref{fig:cauchy}. Although the generated samples matched the center of the Cauchy samples well, in sharp contrast to the observed data, there were no outlying values. We observe the same pattern when inspecting samples generated from the other fitted generative adversarial networks (Figure \ref{fig:sensitivity}) and the diffusion model (Figure \ref{fig:diffusion}). Our simulation results thus agree well with our theory that these deep generative models are unable to capture heavy tails in the target distribution. We also attempted to fit a Gaussian variational autoencoder to the Cauchy data but were unable to get the training to converge. Even when the learning rate was set to extremely small values (such as $1\mathrm{e}-8$), the training loss often fluctuates by orders of magnitude over epochs. Practitioners often report numerical instability when training variational autoencoders \citep{child2021Very, dehaene2021Re, rybkin2021Simple}.  

\begin{figure} 
    \centering
    \begin{subfigure}[t]{0.47\textwidth}
        \centering
        \includegraphics[width=1\textwidth]{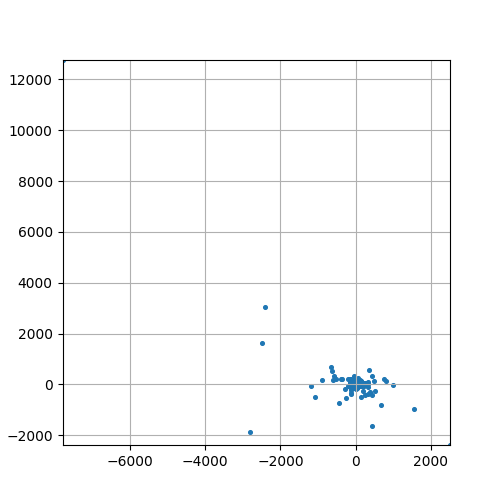}
        \caption{Samples from bivariate Cauchy distribution centred at $0$ with identity scale matrix}
    \end{subfigure}\;
    \begin{subfigure}[t]{0.47\textwidth}
        \centering
        \includegraphics[width=1\textwidth]{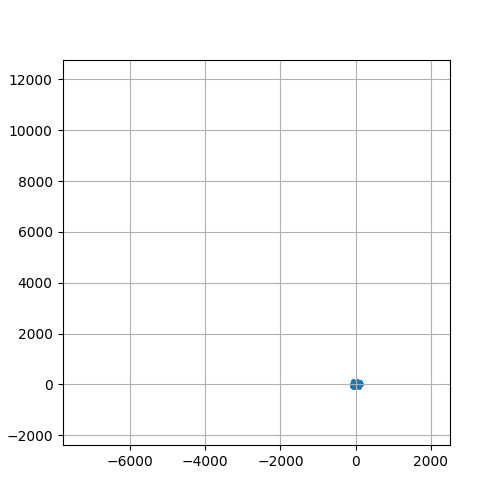}
        \caption{Samples from fitted generative adversarial network}
    \end{subfigure}\;

    \caption{Comparisons between Cauchy samples and synthetic samples from a generative adversarial network.}    \label{fig:cauchy}
\end{figure}

Next, we analyzed data on the Standard and Poor's 500 and the Dow Jones Industrial Average indices from Yahoo Finance. We computed daily returns in basis points for both indices from January $2008$ to April $2024$, totaling $4096$ data points. We then trained a generative adversarial network using these data. The generator has four layers and $64$-dimensional standard Gaussian latent variables. We overlay $4096$ samples of the generated returns with the actual returns in Figure \ref{fig:return}.  The actual daily returns from the Standard and Poor's and Dow Jones indices are positively correlated with each other. The generated returns were able to capture this correlation well. Financial returns are well known to be heavy-tailed. We take the magnitudes of actual and generated returns and inspect them on a log-log plot. Observe that the generated returns are much more concentrated than the actual returns in Figure \ref{fig:returnloglog}. 

In both the simulated and financial data setting, despite the samples being only $2$ dimensional, samples from the fitted generative networks with $64$ dimensional latent variables were unable to capture tail values and generally underestimated uncertainty.

\begin{figure}
    \centering
    \begin{subfigure}[t]{0.48\textwidth}
        \centering
        \includegraphics[width=1\textwidth]{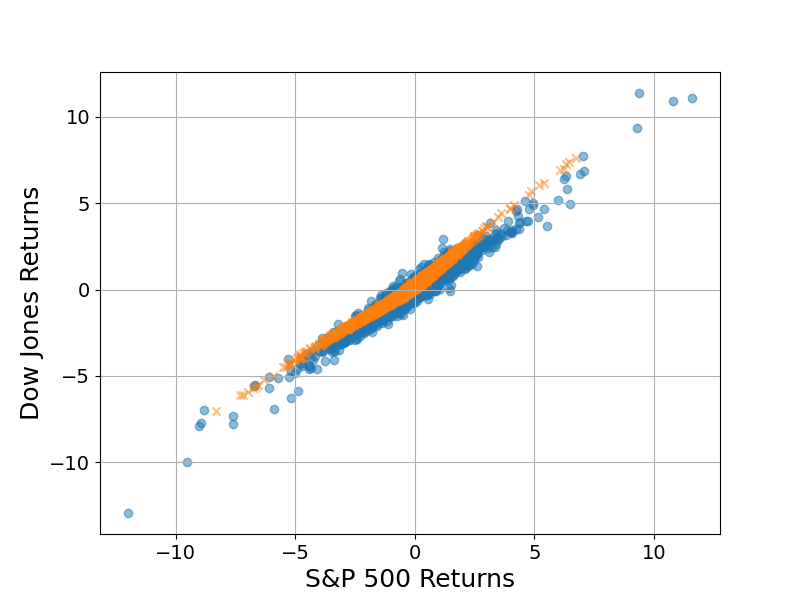}
        \caption{Actual and Synthetic returns for Standard and Poor's $500$ and Dow Jones Industrial Average}
        \label{fig:return}
    \end{subfigure}\;
    \begin{subfigure}[t]{0.48\textwidth}
        \centering
        \includegraphics[width=1\textwidth]{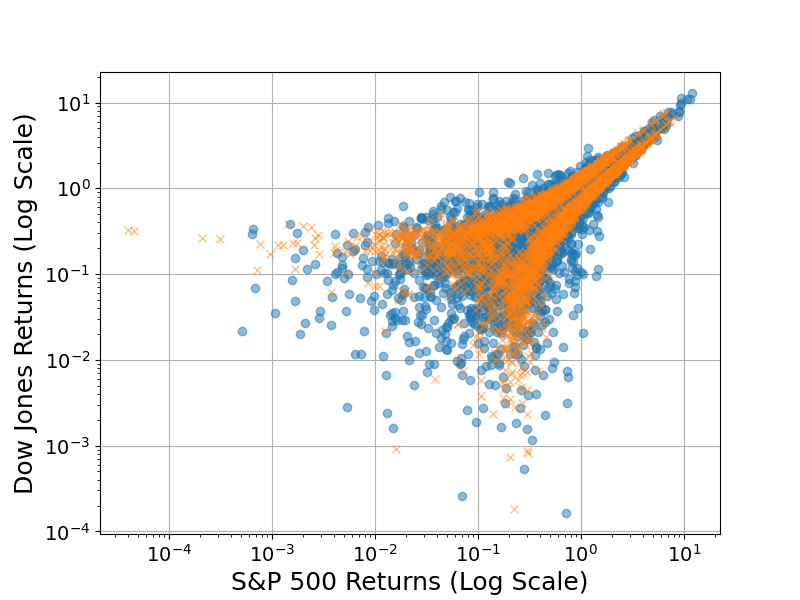}
        \caption{Log-log plot of actual and synthetic return magnitudes for Standard and Poor's $500$ and Dow Jones Industrial Average}
        \label{fig:returnloglog}
    \end{subfigure}\;

    \caption{Comparisons between actual returns from Standard and Poor's 500 and Dow Jones Industrial Average indices versus synthetic samples from a generative adversarial network.}   
\end{figure}

\section{Discussion}
The literature on deep generative models is vast and rapidly evolving. The general framework outlined in this article can be used to analyze other generative models that push forward Gaussian and log-concave latent variables, for example, flow-based models, as long as one can show that the overall push-forward mapping is Lipschitz. 

One focus of our work is on results that are dimension-free or have small dependence on the latent variable dimension $d$. It is natural to consider applying our framework to sub-Gaussian latent variables. It can be shown that, in general, dimension-free Lipschitz concentration results are not attainable for sub-Gaussian random vectors \citep{boucheron2013concentration, ledoux2013probability}. A celebrated inequality due to Talagrand \citep{ledoux1997talagrand, talagrand1996new} shows that if an additional convexity constraint is imposed on $\hat{f}$, a dimension-free bound can be attained. Such convexity assumptions are not appropriate for deep neural networks. 

Contrary to the widespread practice of defaulting to Gaussian latent variables in deep generative models, our work indicates that the choice of latent variable distribution plays a crucial role in applications. It is of great interest to develop more sophisticated priors for these models that allow them to handle heavier-tailed data in finance, anomaly detection, and beyond. Another promising direction is to develop alternative push-forward generative models that go beyond Lipschitz transformations. 

\section*{Acknowledgement}

This work was partially supported by the National Science Foundation, the Office of Naval Research, the National Institutes of Health, and the Warren Alpert Foundation. 
 
\section*{Supplementary material}
\label{SM}
 
\appendix

\section{Preliminaries on Concentration}

\subsection{Sub-Gaussian and sub-exponential random vectors}

In the section, we review the definitions of sub-Gaussian and sub-exponential random variables and random vectors, as well as their various equivalent characterizations.  
\begin{definition}[Sub-Gaussian random variable]
    A real-valued random variable $z$ with mean $E(z) = \mu$ is sub-Gaussian if there exists a constant $C_1 > 0$ such that
\begin{equation*}
    Pr(|z - \mu| \geq t) \leq 2\exp(-t^2/C_1^2)
\end{equation*}
for all $t \geq 0$. 
\end{definition}

The following equivalent characterization will be useful:

\begin{proposition}[Sub-Gaussianity via Orlicz norm]
    A real-valued random variable $z$ with mean $E(z) = \mu$ is sub-Gaussian if and only if there exist constant $C_2 > 0$ such that 
\begin{equation*}
    E[\exp\{(z - \mu)^2/C_2^2\}] \leq 2 
\end{equation*}
The smallest such constant $\inf_{C_2 > 0} E[\exp\{(z - \mu)^2/C_2^2\}] \leq 2 $ is an Orlicz norm of $z$, denoted $||z - \mu||_{\psi_2}$. In other words, $z$ is sub-Gaussian if and only if $||z - \mu||_{\psi_2}$ is finite. 
\end{proposition}

\begin{proof}
    This standard result is detailed in section 2.5.2 of \citet{vershynin2018high}. 
\end{proof}

It is known that the constants $C_1$ and $C_2$ above are equivalent to each other up to universal constants \citep{vershynin2018high}.  
The definition of sub-Gaussianity can be extended to random vectors.
\begin{definition}[Sub-Gaussian random vectors]
A random vector $z \in \R^d$ with mean $E(z) = \mu$ is sub-Gaussian if $\sup_{u \in S^{d - 1}}||u^T (z - \mu)||_{\psi_2} < \infty$, where $S^{d-1}$ is the unit sphere. 
\end{definition}

There is a related weaker notion of sub-exponentiality that we review below. 

\begin{definition}[Sub-exponential random variable]
    A real-valued random variable $z$ with mean $E(z) = \mu$ is sub-exponential if there exists a constant $C_3 > 0$ such that
\begin{equation*}
    Pr(|z - \mu| \geq t) \leq 2\exp(-t/C_3)
\end{equation*}
for all $t \geq 0$. 
\end{definition}

\begin{proposition}[Sub-exponentiality via Orlicz norm]
    A real-valued random variable $z$ with mean $E(z) = \mu$ is sub-exponential if and only if there exist constant $C_4 > 0$ such that 
\begin{equation*}
    E\{\exp(|z - \mu|/C_4)\} \leq 2 
\end{equation*}
The smallest such constant $\inf_{C_4 > 0} E\{\exp(|z - \mu|/C_4)\} \leq 2 $ is an Orlicz norm of $z$, denoted $||z - \mu||_{\psi_1}$. In other words, $z$ is sub-exponential if and only if $||z - \mu||_{\psi_1}$ is finite. 
\end{proposition}

\begin{proof}
    This standard result is detailed in section 2.7 of \citet{vershynin2018high}. 
\end{proof}

It is known that the constants $C_3$ and $C_4$ above are equivalent to each other up to universal constants \citep{vershynin2018high}.  
The notion can be extended to random vectors. 
\begin{definition}[Sub-exponential random vector]
A random vector $z \in \R^d$ with mean $E(z) = \mu$ is sub-exponential if $\sup_{u \in S^{d - 1}}||u^T (z - \mu)||_{\psi_1} < \infty$, where $S^{d-1}$ is the unit sphere. 
\end{definition}

\subsection{Lipschitz concentration of random vectors}

We are interested in studying the distributional properties of $\hat{f}(z)$, where $\hat{f}$ is a Lipschitz function and $z$ is sub-Gaussian or log-concave. 
We start with the Gaussian isoperimetric inequality. 

\begin{theorem}
[Gaussian isoperimetric inequality] 
    Let $z$ be a standard multivariate Gaussian random vector in $\R^d$. Let $f:\R^d \to \R$ be a real-valued, $\mathcal{L}$-Lipschitz function with respect to the Euclidean norm. Then the following concentration inequality holds: 
 $$Pr[|\hat{f}(z) - E\{\hat{f}(z)\}| \geq t] \leq 2\exp\{-t^2/(2\mathcal{L}^2)\}$$ for all $t \geq 0$.
\end{theorem}

\begin{proof}
    This inequality is due to \citet{sudakov1978extremal} and \citet{borell1975brunn}. The version used here as well as the proof can be found in Theorem 2.26 of \citet{wainwright2019high}.
\end{proof}

While the above result is for standard multivariate Gaussian $z$, as a simple corollary, we can obtain similar results for $z'$ that follow a general multivariate Gaussian distribution with mean $\mu$ and covariance $\Sigma$. Observe that $z'$ can be written as $ \Sigma^{1/2}z + \mu$, a linear transformation with Lipschitz constant $||\Sigma^{1/2}||$, where $||\cdot||$ denotes the spectral norm. Composing Lipschitz functions, and observing that $||\Sigma^{1/2}||^2 = ||\Sigma||$, we obtain \begin{equation*}
    Pr[|\hat{f}(z') - E\{\hat{f}(z')\}| \geq t] \leq 2\exp\{-t^2/(2\mathcal{L}^2||\Sigma||)\} 
\end{equation*} for all $t \geq 0$.

Similar Lipschitz transformation results for log-concave and strongly log-concave distributions are available from the high dimensional geometry literature.  

\begin{theorem}[Log-concave Lipschitz concentration] \label{logconcave_lipschitz}
Let $z \in \R^d$ be a random vector with isotropic log-concave probability density, so that it is centred and has identity covariance. Let $f: \R^d \to \R$ be a $\mathcal{L}$-Lipschitz function. Then 
$$\Pr[|\hat{f}(z) - E\{\hat{f}(z)\}| > t\mathcal{L}] \leq \exp(-C_5 \Psi_z t)$$
for some absolute constant $C_5> 0$, where $\Psi_z$ is the Cheeger's constant of the density of $z$. 
\end{theorem}
\begin{proof}
 The above theorem is due to \cite{gromov1983topological}. We adopt the form used in \cite{lee2018kannan}. In the literature $\Psi_z$ is sometimes defined by the reciprocal of the definition we used here. We adopt the definition in \cite{lee2018kannan} for consistency. 
\end{proof}

By reparametrizing $t'= t\mathcal{L}$ and rewriting $C_5$ as $1/C_6$, we can obtain an equivalent inequality, which
we use repeatedly,
\begin{align*}
    \Pr[|\hat{f}(z) - E\{\hat{f}(z)\}| > t'] \leq \exp\{-\Psi_z t'/(C_6\mathcal{L})\}.
\end{align*}
The isotropic and centered conditions can be dropped by considering $\Sigma^{1/2}z + \mu$ for general covariance $\Sigma$ and mean $\mu$, with the Lipschitz constant in the upper bound then increasing by a factor of $||\Sigma^{1/2}||$. 

We also consider Lipschitz transformations of strongly log-concave random vectors. The definition of $\gamma$-strongly log-concave distributions can be found in chapter 3 of \cite{wainwright2019high}. 

\begin{theorem}
[Strongly Log-concave Lipschitz concentration]\label{strong_logconcave}
    Let $z$ be a $\gamma$-strongly log-concave random vector in $\R^d$. Let $\hat{f}:\R^d \to \R$ be a real-valued, $\mathcal{L}$-Lipschitz function with respect to the Euclidean norm. Then the following concentration inequality holds. 
 $$Pr[|\hat{f}(z) - E\{\hat{f}(z)\}| \geq t] \leq 2\exp\{-\gamma t^2/(4\mathcal{L}^2)\}$$ for all $t \geq 0$.
\end{theorem}
\begin{proof}
    The above standard result is directly adapted from theorem 3.16 of \cite{wainwright2019high}.
\end{proof}

We will also use the Gromov--Levy inequality to study concentration on certain manifolds. 
 
\begin{theorem}[Gromov--Levy] \label{gromov_levy}
Let $(M, g)$ be a compact, connected $d$-dimensional smooth Riemannian manifold with $d \geq 2$. Use $\lambda$ to denote the infimum of the Ricci curvature tensor evaluated over any pair of unit tangent vectors associated with any point on the manifold. Assume $\lambda > 0$. Let $\nu$ be its normalized volume element and $z \sim \nu$.  Let $h: M \to \R$ be a $\mathcal{L}$-Lipschitz function. Then  
$$Pr[|h(z) - E\{h(z)\}| \geq t] \leq 2\exp\{-\lambda t^2/(2\mathcal{L}^2)\}$$
 \end{theorem}

\begin{proof}
This result is due to \cite{gromov1986isoperimetric}, and we use the version provided in proposition 2.17 of \cite{ledoux2001concentration}, which states this inequality in one-sided form for $1$-Lipschitz functions. By considering $-f$ and using the union bound, we obtain the two sided version above for $1$-Lipschitz functions. Apply proposition 1.2 in \cite{ledoux2001concentration} to get the general $\mathcal{L}$-Lipschitz case above. 
\end{proof}

We also state the following useful lemma: 

\begin{lemma} \label{projection_lipschitz}
The projection map $h_{j_1, \ldots j_k}: \R^p \to \R^k$ defined by $(x_1, \ldots x_p) \mapsto (x_{j_1}, \ldots, x_{j_k})$ for any $j_1, \ldots, j_k \in [p]$ is $1$-Lipschitz with respect to the Euclidean norm. 
\end{lemma}
\begin{proof}
    Follows directly from the fact that $||x - y||_2 = \sqrt{\sum_{i = 1}^p (x_i - y_i)^2} \geq \sqrt{\sum_{q = 1}^k (x_{j_q} - y_{j_q})^2} = ||h_{j_1, \ldots j_k}(x) - h_{j_1, \ldots j_k}(y)||_2$ for any $x, y \in \R^p$. 
\end{proof}
Given a vector in $\R^p$, the projection map above simply selects its $j_1, \ldots, j_k$ components to return a vector in $\R^k$. 


\section{Proofs}
\subsection{Proof of Proposition \ref{finite_lipschitz}}
\begin{proof}
A finite neural network consists of finitely many compositions of affine transformations by $W_l$ and $b_l$ and non-linear activations by $\sigma_l$. 
Since the Lipschitz property is preserved under finite composition, and since $\sigma_l \in S$, it suffices to show that the affine transformations are Lipschitz. The Lipschitz constant of the affine function $g(x) = W_l x + b_l$ is upper bounded by the Frobenius norm $ ||W_l||_F$, which is finite since matrix entries and dimensions are finite.   
\end{proof}

\subsection{Proof of Theorem \ref{gaussian_iso}}

\begin{proof}
We break down function $\hat{f}: \R^d \to \R^p$ into its $p$ component functions $\hat{f}_1, \ldots, \hat{f}_p: \R^d \to \R$,
which are also $\mathcal{L}$-Lipschitz by Lemma \ref{projection_lipschitz}. Focus on $\hat{f}_1$ without loss of generality. Apply the Gaussian isoperimetric inequality to get that $\hat{f}_1(z) - E\{\hat{f}_1(z)\}$ is a sub-Gaussian random variable with Orlicz norm $||\hat{f}_1(z) - E\{\hat{f}_1(z)\}||_{\psi_2} \leq c\mathcal{L}||\Sigma||^{1/2}$ for some $c > 0$. This holds for all $\hat{f}_1, \ldots, \hat{f}_p$. Apply Lemma \ref{subGaussian_components} to see that $\hat{f}(z) - E\{\hat{f}(z)\}$ is a sub-Gaussian random vector with $||u^T[\hat{f}(z) - E\{\hat{f}(z)\}]||_{\psi_2} \leq c\sqrt{p}\mathcal{L}||\Sigma||^{1/2}$ for any $u \in S^{p-1}$. This implies the concentration inequality 
$Pr[u^T |\hat{f}(z) - E\{\hat{f}(z)\}| \geq t] \leq 2\exp\{- t^2/(C^2p\mathcal{L}^2||\Sigma||)\}$ for any unit vector $u \in S^{p - 1}$ and some $C > 0$. Substituting $C_p^2 = C^2p\mathcal{L}^2||\Sigma||$ into the bound yields the desired result.  
\end{proof}


\begin{lemma} \label{subGaussian_components}
Given $p$ sub-Gaussian random variables $x_1, \ldots, x_p$, if $||x_i||_{\psi_2} \leq K$ for some positive $K$ and for any $ 1\leq i \leq p$, then $x = (x_1, \ldots, x_p)$ is sub-Gaussian and $\sup_{u \in S^{p-1}} || u^Tx||_{\psi_2} \leq \sqrt{p}K$.   
\end{lemma}
\begin{proof}
Simply observe that $\sup_{u \in S^{p-1}} || u^Tx||_{\psi_2}=\sup_{u \in S^{p-1}} || \sum_{i = 1}^p u_i x_i||_{\psi_2} \leq \sup_{u \in S^{p-1}}  \sum_{i = 1}^p (|u_i| \cdot ||  x_i||_{\psi_2}) \leq K \sup_{u \in S^{p-1}} ||u||_1 \leq \sqrt{p}K$. Here the first inequality is due to the triangle inequality and homogeneity, the last inequality is due to the inequality $||u||_1 \leq \sqrt{p}||u||_2$.
\end{proof}

\subsection{Proof of Theorem \ref{logconcave_iso}}

\begin{proof}
We again break down $\hat{f}: \R^d \to \R^p$ into $p$ component functions $\hat{f}_1, \ldots, \hat{f}_p: \R^d \to \R$, which are also $\mathcal{L}$-Lipschitz by Lemma \ref{projection_lipschitz}. Focus on $\hat{f}_1$ without loss of generality. Apply the log-concave concentration inequality in Lemma \ref{logconcave_lipschitz} to get that $\hat{f}_1(z) - E(\hat{f}_1(z))$ is a sub-exponential random variable with Orlicz norm $||\hat{f}_1(z) - E(\hat{f}_1(z))||_{\psi_1} \leq c\mathcal{L}||\Sigma^{1/2}||/\Psi_z$ for some $c > 0$. Apply Lemma \ref{subexponential_components} to see that $\hat{f}(z) - E\{\hat{f}(z)\}$ is a sub-exponential random vector with $\sup_{u \in S^{p-1}}||u^T[\hat{f}(z) - E\{\hat{f}(z)\}]||_{\psi_1} \leq c\sqrt{p}\mathcal{L}||\Sigma^{1/2}||/\Psi_z$. This implies the concentration inequality 
$Pr[u^T |\hat{f}(z) - E\{\hat{f}(z)\}| \geq t] \leq 2\exp\{- \Psi_z t/(C\sqrt{p}\mathcal{L}||\Sigma^{1/2}||)\}$ for any unit vector $u \in S^{p - 1}$ and some  $C > 0$. Substituting $C_p = C\sqrt{p}\mathcal{L}||\Sigma^{1/2}||/\Psi_z$ into the bound yields the desired result.  
\end{proof}

\begin{lemma} \label{subexponential_components}
Given $p$ sub-exponential random variables $x_1, \ldots, x_p$, if $||x_i||_{\psi_1} \leq K$ for some positive $K$ and for any $ 1\leq i \leq p$, then $x = (x_1, \ldots, x_p)$ is sub-exponential and $\sup_{u \in S^{p-1}} || u^Tx||_{\psi_1} \leq \sqrt{p}K$.   
\end{lemma}
\begin{proof}
Simply observe that $\sup_{u \in S^{p-1}} || u^Tx||_{\psi_1}=\sup_{u \in S^{p-1}} || \sum_{i = 1}^p u_i x_i||_{\psi_1} \leq \sup_{u \in S^{p-1}}  \sum_{i = 1}^p (|u_i| \cdot ||  x_i||_{\psi_1}) \leq  K \sup_{u \in S^{p-1}} ||u||_1 \leq \sqrt{p}K$. Here the first inequality is due to the triangle inequality and homogeneity, the last inequality is due to the inequality $||u||_1 \leq \sqrt{p}||u||_2$.
\end{proof}

\subsection{Proof of Theorem \ref{strong_logconcave_iso}}

\begin{proof}
We yet again break $\hat{f}: \R^d \to \R^p$ into component functions $\hat{f}_1, \ldots, \hat{f}_p: \R^d \to \R$. Each of which is also $\mathcal{L}$-Lipschitz by lemma \ref{projection_lipschitz}. Focus on $\hat{f}_1$ without loss of generality. The strong log-concave concentration inequality implies $\hat{f}_1(z) - E\{\hat{f}_1(z)\}$ is a sub-Gaussian random variable with Orlicz norm $||\hat{f}_1(z) - E\{\hat{f}_1(z)\}||_{\psi_2} \leq c\mathcal{L}||\Sigma||^{1/2}/\sqrt{\gamma}$ for some $c > 0$. This holds for all component functions $\hat{f}_1, \ldots, \hat{f}_p$. From Lemma \ref{subGaussian_components} $\hat{f}(z) - E\{\hat{f}(z)\}$ is a sub-Gaussian random vector with $||u^T[\hat{f}(z) - E\{\hat{f}(z)\}]||_{\psi_2} \leq c\sqrt{p}\mathcal{L}||\Sigma||^{1/2}/\sqrt{\gamma}$ for any $u \in S^{p-1}$. This implies the concentration inequality 
$Pr[u^T |\hat{f}(z) - E\{\hat{f}(z)\}| \geq t] \leq 2\exp\{- \gamma t^2/(C^2p\mathcal{L}^2||\Sigma||)\}$ for any unit vector $u \in S^{p - 1}$ and some $C > 0$. Substituting $C_{p, \gamma}^2 = C^2\sqrt{p}\mathcal{L}^2||\Sigma||/\gamma$ into the bound yields the desired result.  
\end{proof}

\subsection{Proof of Theorem \ref{manifold_iso}}

\begin{proof}
For ease of notation use $h$ to denote the function $\hat{f}\circ \phi: M \to \R^p$. Since $\hat{f}$ and $\phi$ are  $\mathcal{L}$- and $\mathcal{L}_\phi$- Lipschitz respectively, by composition of Lipschitz functions $h$ is $\mathcal{L}\mathcal{L}_\phi$-Lipschitz. 
Break down the function $h: M \to \R^p$ into its $p$ component functions $h_1, \ldots, h_p$, where each component maps $M \to \R$ and is also $\mathcal{L}\mathcal{L}_\phi$ Lipschitz with respect to the geodesic distance on $M$ by Lemma \ref{projection_lipschitz}. We first focus on $h_1$ without loss of generality. Apply Theorem \ref{gromov_levy} to get that $h_1(z) - E\{h_1(z)\}$ is a sub-Gaussian random variable with Orlicz norm $||h_1(z) - E\{h_1(z)\}||_{\psi_2} \leq c\mathcal{L}\mathcal{L}_\phi/\sqrt{\lambda}$ for some $c > 0$. This holds for all component functions $h_1, \ldots, h_p$. By Lemma \ref{subGaussian_components} $\hat{f}(z) - E\{\hat{f}(z)\}$ is a sub-Gaussian random vector with $||u^T[\hat{f}(z) - E\{\hat{f}(z)\}]||_{\psi_2} \leq c\sqrt{p}\mathcal{L}\mathcal{L}_\phi/\sqrt{\lambda}$ for any $u \in S^{p-1}$. This implies the concentration inequality 
$Pr[u^T |\hat{f}(z) - E\{\hat{f}(z)\}| \geq t] \leq 2\exp\{- \lambda t^2/(C^2p\mathcal{L}^2\mathcal{L}_\phi^2)\}$ for any unit vector $u \in S^{p - 1}$ and some $C > 0$. Substituting $C_{p, \lambda}^2 = C^2p\mathcal{L}^2\mathcal{L}_\phi^2/\lambda$ into the bound yields the desired result.  
\end{proof}

\section{Reduction Argument for Diffusion Models and Proof} \label{reduction}
The goal of the reduction is to turn the sequence of transformations dictated by the iterative step described in equation \eqref{iterative} into a single transformation. More precisely, we want to write the sample $X_0$ as a single Lipschitz transformation of a standard Gaussian random vector in order to derive concentration results using the Gaussian isoperimetric inequality. The complication is that Gaussian noise $\epsilon_\tau\sigma_\tau$ is added at each time step $\tau > 1$, so we cannot directly write $X_0$ as a deterministic Lipschitz transformation of $X_T = z \sim N_p(0, I)$ due to the extra randomness that we have to account for. 

The key insight is to realize that we can write $X_0$ as a single deterministic Lipschitz transformation of the augmented Gaussian random vector $(X_T, \epsilon_1, \ldots, \epsilon_T)$. While this is a $N_{p(T + 1)}(0, I)$ random vector, the large dimension $p(T+1)$ does not directly enter the subsequent bound due to the dimension-free nature of the Gaussian isoperimetric inequality. We detail these steps below. 

Define an auxiliary process $(Y_\tau)_{\tau = 0}^T = \{(X_\tau, \epsilon_1, \ldots, \epsilon_T)\}_{\tau = 0}^T$ where we concatenate the original process $X_\tau$ with the entire noise sequence. At time $T$, $Y_T$ is a $N_{p(T + 1)}(0, I)$ random vector. As $\tau$ evolves from $T$ to $0$, only the first $p$ components of $Y_\tau$ corresponding to $X_\tau$ are updated, while the rest of the components corresponding to noise remain unchanged. We use superscripts such as $Y^{1:p}_{\tau}$ to denote the first $p$ components of $Y_\tau$. 
We now rewrite the iterative step in equation $\eqref{iterative}$ as follows:
\begin{align*}
Y_{\tau - 1} = J_\tau(Y_{\tau}) + K_\tau(Y_{\tau}),
\end{align*}
where $J_\tau$ and $K_\tau$ are functions mapping from $\R^{p(T + 1)}$ to $\R^{p(T + 1)}$. We write $J_\tau$ as $(J^{1:p}_\tau, \textnormal{Id}_{pT})$ and $K_\tau$ as $(K^{1:p}_\tau, \textnormal{Id}_{pT})$ to emphasize that $J_\tau$ and $K_\tau$ only act on the first $p$ components. $J^{1:p}_\tau$ and $K^{1:p}_\tau$ map from $\R^{p(T + 1)}$ to $\R^p$ while function
$\textnormal{Id}_{pT}:\R^{p(T + 1)}\to \R^{pT}$ is defined by the mapping $\{x_1, \ldots, x_{p(T + 1)}\} \mapsto \{x_{p + 1}, \ldots, x_{p(T + 1)}\}$, and hence $J_\tau$ and $K_\tau$ are mappings $\{x_1, \ldots, x_{p(T + 1)}\} \mapsto [J^{1:p}_\tau\{x_1, \ldots, x_{p(T + 1)}\}, x_{p + 1}, \ldots, x_{p(T + 1)}]$ and $[K^{1:p}_\tau\{x_1, \ldots, x_{p(T + 1)}\}, x_{p + 1}, \ldots, x_{p(T + 1)}]$ respectively. On the first $1:p$ components, $Y^{1:p}_\tau$ is updated by 
\begin{align*}
        Y^{1:p}_{\tau - 1} = J^{1:p}_\tau(Y_{\tau}) + K^{1:p}_\tau(Y_{\tau}),
\end{align*}
where $J_\tau^{1:p}(Y_{\tau}) = (1/\sqrt{\alpha_\tau})\{X_{\tau} - (1-\alpha_\tau)/\sqrt{1-\bar{\alpha}_\tau}\hat{f}(X_{\tau}, \tau)\}$ is a function that depends on only the first $p$ components of $Y_\tau$, which are $X_\tau$, and $K_\tau^{1:p}(Y_{\tau}) = \epsilon_\tau \sigma_{\tau} \text{1}_{\tau > 1}$ is a function that depends on only the components of $Y_\tau$ that corresponds to $\epsilon_\tau$. 

We analyze Lipschitz properties of these functions. Here $\tau$, $\alpha_\tau$, $\bar{\alpha}_\tau$, $\sigma_\tau$ are deterministic quantities. $J_\tau^{1:p}$ is simply the addition of a scaled $X_\tau$ to a scaled finite neural network function. By Lemma \ref{augmentation_lipschitz}, $\hat{f}(X_{\tau}, \tau)$ is a Lipschitz function of $X_\tau$ since $\hat{f}$ is a finite neural network. Since the Lipschitz property is closed under addition and finite scaling, $J_\tau^{1:p}$ is then a Lipschitz function of $X_\tau$. By Lemma \ref{projection_lipschitz}, $J_\tau^{1:p}$ is a Lipschitz function of $Y_\tau$. By Lemma \ref{ID_lipschitz}, $J_\tau$ is a Lipschitz function of $Y_\tau$. 
$K_\tau^{1:p}$ is simply a scaling of $\epsilon_\tau$, and hence is a Lipschitz function of $\epsilon_\tau$. By Lemma \ref{projection_lipschitz}, $K_\tau^{1:p}$ is a Lipschitz function of $Y_\tau$. By Lemma \ref{ID_lipschitz}, $K_\tau$ is a Lipschitz function of $Y_\tau$. 
Hence $J_\tau$, $K_\tau$ and consequently $J_\tau + K_\tau$ are Lipschitz functions of $Y_\tau$. 

Let $H_\tau = J_\tau + K_\tau$ with Lipschitz constant $\mathcal{L}_\tau$. 
Observe that $Y_0$ is equal to $H_1[H_2 \{\ldots H_T(Y_T)\ldots\}]$, and the overall transformation is $\prod_{\tau = 1}^T \mathcal{L}_\tau$ Lipschitz. By Lemma \ref{projection_lipschitz}, $X_0$ is then also a Lipschitz function of $Y_T$ with the same Lipschitz constant. 
We have thus shown that $X_0$ can be written as a $\prod_{\tau = 1}^T \mathcal{L}_\tau$-Lipschitz function of $Y_T$, a $N_{(T + 1)p}(0, I)$ random vector. We are ready to state the proof of Theorem \ref{diffusion}.

\subsection{Proof of Theorem \ref{diffusion}}

\begin{proof}
We have demonstrated above that $X_0$ can be written as a $\prod_{\tau = 1}^T \mathcal{L}_\tau$-Lipschitz function of $Y_T$, a $N_{(T + 1)p}(0, I)$ random vector. Use $h$ to denote this function, so $X_0 = h(Y_T)$. 
Break down the function $h: \R^{p(T + 1)}\to \R^p$ into $p$ component functions $h_1, \ldots, h_p$, where each component is a function from $\R^{p(T + 1)}$ to $\R$ that is also $\prod_{\tau = 1}^T \mathcal{L}_\tau$ Lipschitz by Lemma \ref{projection_lipschitz}. Focus on $h_1$ without loss of generality. Apply the Gaussian isoperimetric inequality to get that $h_1(Y_T) - E\{h_1(Y_T)\}$ is a sub-Gaussian random variable with Orlicz norm $||h_1(Y_T) - E\{h_1(Y_T)\}||_{\psi_2} \leq c\prod_{\tau = 1}^T \mathcal{L}_\tau$ for some $c > 0$. This holds for all component functions $h_1, \ldots, h_p$. By Lemma \ref{subGaussian_components}  $h(Y_T) - E\{h(Y_T)\}$ is a sub-Gaussian random vector with $||u^T[h(Y_T) - E\{h(Y_T)\}]||_{\psi_2} \leq c\sqrt{p}\prod_{\tau = 1}^T \mathcal{L}_\tau$ for any $u \in S^{p-1}$. This implies the concentration inequality 
$Pr[u^T |h(Y_T) - E\{h(Y_T)\}| \geq t] \leq 2\exp\{-t^2/(C^2p\prod_{\tau = 1}^T \mathcal{L}^2_\tau)\}$ for any unit vector $u \in S^{p - 1}$ and some $C > 0$. Substituting $C_{p}^2 = C^2p\prod_{\tau = 1}^T \mathcal{L}^2_\tau$ into the bound yields the desired result.  
\end{proof}

We also state the following useful lemmas: 
\begin{lemma} \label{augmentation_lipschitz}
The map $h_{\tau}: \R^p \to \R^{p + 1}$ defined by $(x_1, \ldots, x_p) \mapsto (x_1, \ldots, x_p, \tau)$ for some fixed deterministic real number $\tau$ is $1$-Lipschitz with respect to the Euclidean norm. 
\end{lemma}
\begin{proof}
     $||x - y||_2 = \sqrt{\sum_{i = 1}^p (x_i - y_i)^2} = \sqrt{\sum_{i = 1}^p (x_i - y_i)^2 + (\tau - \tau)^2} = ||h_\tau(x) - h_\tau(y)||_2$ for any $x, y \in \R^p$. 
\end{proof}

\begin{lemma} \label{ID_lipschitz}
If the map $h: \R^{p + k} \to \R^p$ is $\mathcal{L}$-Lipschitz with respect to the Euclidean norm, then the map $(h, \textnormal{Id}_k): \R^{p + k} \to \R^{p + k}$ defined by $(x_1, \ldots, x_{p + k}) \mapsto \{h(x_1, \ldots, x_{p + k}), x_{p + 1}, \ldots, x_{p + k}\}$ is $(\mathcal{L} + 1)$-Lipschitz for any integer $k \geq 0$.  
\end{lemma}
\begin{proof}
The case of $k = 0$ is trival. 
For $k \geq 1$, simply note that $$||\{h(x_1, \ldots, x_{p + k}), x_{p + 1}, \ldots, x_{p + k}\} - \{h(y_1, \ldots, y_{p + k}), y_{p + 1}, \ldots, y_{p + k}\}||_2$$
$$ = \sqrt{||h(x_1, \ldots, x_{p + k}) - h(y_1, \ldots, y_{p + k})||_2^2 + \sum_{i = p + 1}^{p + k} (x_i - y_i)^2}$$ 
$$ \leq \sqrt{\mathcal{L}^2||x - y||_2^2 + ||x - y||_2^2} = (\mathcal{L} + 1)||x - y||_2$$

for any $x = (x_1, \ldots, x_{p + 1}), y = (y_1, \ldots, y_{p + k}) \in \R^{p+k}$. 
\end{proof}

\section{Implementation details of Simulations and Experiments}\label{implementation}

Code can be found at \url{https://www.github.com/edrictam/generative_capacity}. All figures in this article are generated via Python notebooks executed in Google Colaboratory. 
\subsection{Details on Generative Adversarial Networks}

In Figure \ref{fig:cauchy}, generative adversarial networks are trained on the bivariate Cauchy data. The discriminator neural network has four fully-connected layers with widths $2, 256, 128, 1$. The generator neural network has four layers with widths $64, 128, 256, 2$. All activation functions are rectified linear units, except the last layer of the discriminator network, which has sigmoidal activation for binary classification. The generative adversarial network is trained over $500$ epochs with batch size $100$ using the Adam optimizer. Learning rates for both networks are set to $0.0002$. Standard binary cross entropy losses are employed. Latent variables used for sample generation follow a $64$-dimensional standard normal distribution. 

In Figure \ref{fig:sensitivity}, we show  samples from four additional generative adversarial networks. These generative adversarial networks share the same architecture and training setup as the network in Figure \ref{fig:cauchy}, except with different number of layers and latent variable dimensions. We used 6-layer and 8-layer generators and discriminators in Figures \ref{fig:deeper} and \ref{fig:deepest} respectively. Here, additional layers are fully-connected with width $256$. We used $32$-dimensional and $128$-dimensional standard Gaussian latent variables in \ref{fig:narrow} and \ref{fig:wide} respectively. 

\subsection{Details on Denoising Diffusion Model}

We trained a denoising diffusion model on the bivariate Cauchy training data. Here, the noise prediction network has four fully-connected layers with dimensions $2 + 1, 128, 128, 2$.  All activation functions are rectified linear units. The network is trained using the Adam optimizer with learning rate $0.001$ over $1000$ epochs with batch size $128$. The number of time steps for the diffusion model is $1000$, with variance schedule $\beta_1, \cdots, \beta_{1000}$ set to be an arithmetic sequence starting at $0.0001$ with increment $0.02$. 

\subsection{Financial Data}
Price data for the Standard and Poor's $500$ as well as the Dow Jones Industrial Average indices were obtained from Yahoo Finance for the period from the first of January, 2008 to the twelfth of April, 2024. Daily closing prices are transformed into daily returns in basis points, yielding a total of $4096$ $2$-dimensional data points. A generative adversarial network is trained on these data. The discriminator neural network has four layers with widths $2, 256, 128, 1$. The generator neural network has four layers with widths $64, 128, 256, 2$. All activation functions are rectified linear units, except the last layer of the discriminator network, which has sigmoidal activation for binary classification. The generative adversarial network is trained over $200$ epochs with batch size $64$ using the Adam optimizer. Learning rate is set to $0.0001$ for the generator network and $0.00005$ for the discriminator network. Standard binary cross entropy losses are employed. Latent variables used for sample generation follow a $64$-dimensional standard normal distribution. 

\section{Additional Figures for Simulations}
\begin{figure} 
\centering
    \begin{subfigure}[t]{0.47\textwidth}
        \centering
        \includegraphics[width=1\textwidth]{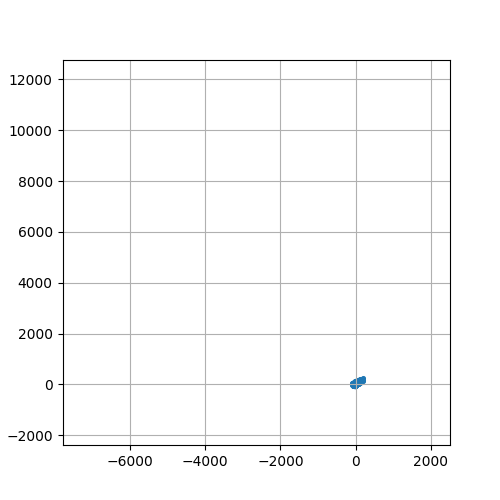}
        \caption{Samples from fitted generative adversarial network (8 layers, 64 latent variables)}
        \label{fig:deepest}
    \end{subfigure}\;
    \begin{subfigure}[t]{0.47\textwidth}
        \centering
        \includegraphics[width=1\textwidth]{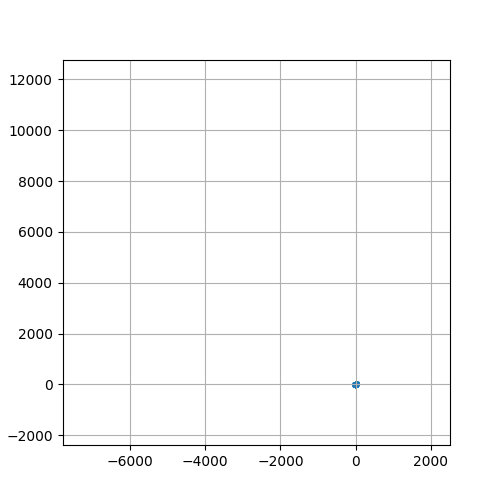}
        \caption{Samples from fitted generative adversarial network (6 layers, 64 latent variables)}
        \label{fig:deeper}
    \end{subfigure}\;
        \begin{subfigure}[t]{0.47\textwidth}
        \centering
        \includegraphics[width=1\textwidth]{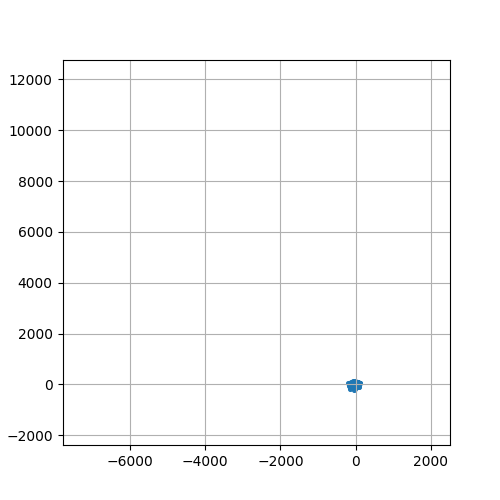}
        \caption{Samples from fitted generative adversarial network (4 layers, 32 latent variables)}
        \label{fig:narrow}
    \end{subfigure}\;
    \begin{subfigure}[t]{0.47\textwidth}
        \centering
        \includegraphics[width=1\textwidth]{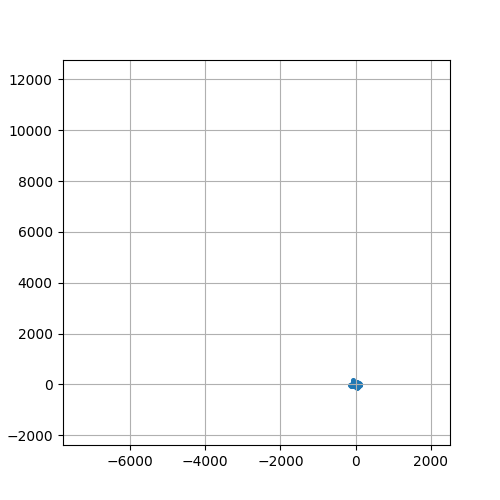}
        \caption{Samples from fitted generative adversarial network (4 layers, 128 latent variables)}
        \label{fig:wide}
    \end{subfigure}\;
    \caption{Synthetic samples from generative adversarial networks with varying depth and latent variable dimensions}    \label{fig:sensitivity}
\end{figure}

\begin{figure} 
    \centering
        \begin{subfigure}[t]{0.47\textwidth}
        \centering
        \includegraphics[width=1\textwidth]{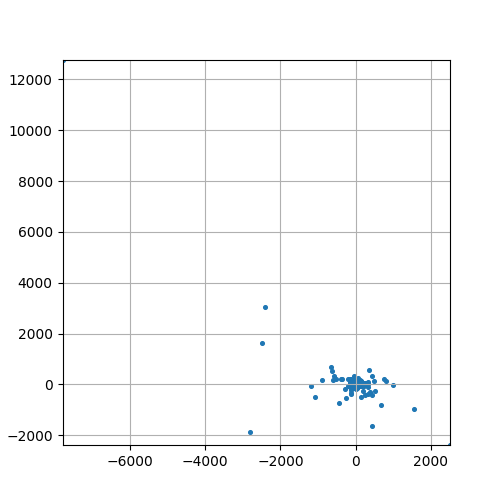}
        \caption{Bivariate Cauchy training samples}
    \end{subfigure}\;
    \begin{subfigure}[t]{0.47\textwidth}
        \centering
        \includegraphics[width=1\textwidth]{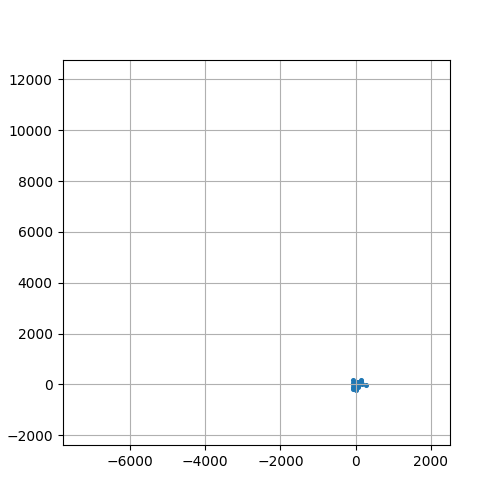}
        \caption{Samples from denoising diffusion model}
    \end{subfigure}\;
        \caption{Comparisons between Cauchy samples and synthetic samples from a denoising diffusion model.}    \label{fig:diffusion}
\end{figure}

\clearpage
\bibliographystyle{biometrika}
\bibliography{ref}

\end{document}